\documentclass{article} 
\usepackage{iclr2027_conference,times}

\usepackage{amsmath,amsfonts,bm}

\def\eqref#1{equation~\ref{#1}}

\def\1{\bm{1}}

\DeclareMathAlphabet{\mathsfit}{\encodingdefault}{\sfdefault}{m}{sl}
\SetMathAlphabet{\mathsfit}{bold}{\encodingdefault}{\sfdefault}{bx}{n}

\usepackage{hyperref}
\usepackage{url}
\usepackage{booktabs}       
\usepackage{amsfonts}       
\usepackage{nicefrac}       
\usepackage{microtype}      
\usepackage{xcolor}         
\usepackage{multirow}
\usepackage{array}
\usepackage{makecell}
\usepackage{subcaption}
\usepackage{graphicx}
\usepackage{amsthm,amsmath,amssymb}
\usepackage{float}
\usepackage{bm}
\usepackage{wrapfig}
\newtheorem{theorem}{Theorem}

\title{Streaming Video Editing with Easy Adaptation}

\author{
Yujia Hu \quad Jiajun Li \quad Zihao He \quad Songhua Liu\thanks{Corresponding author.} \\
Shanghai Jiao Tong University
}

\iclrfinalcopy 
\begin{document}

\maketitle

\begin{figure}[ht]
    \centering 
    \includegraphics[width=1.0\textwidth]{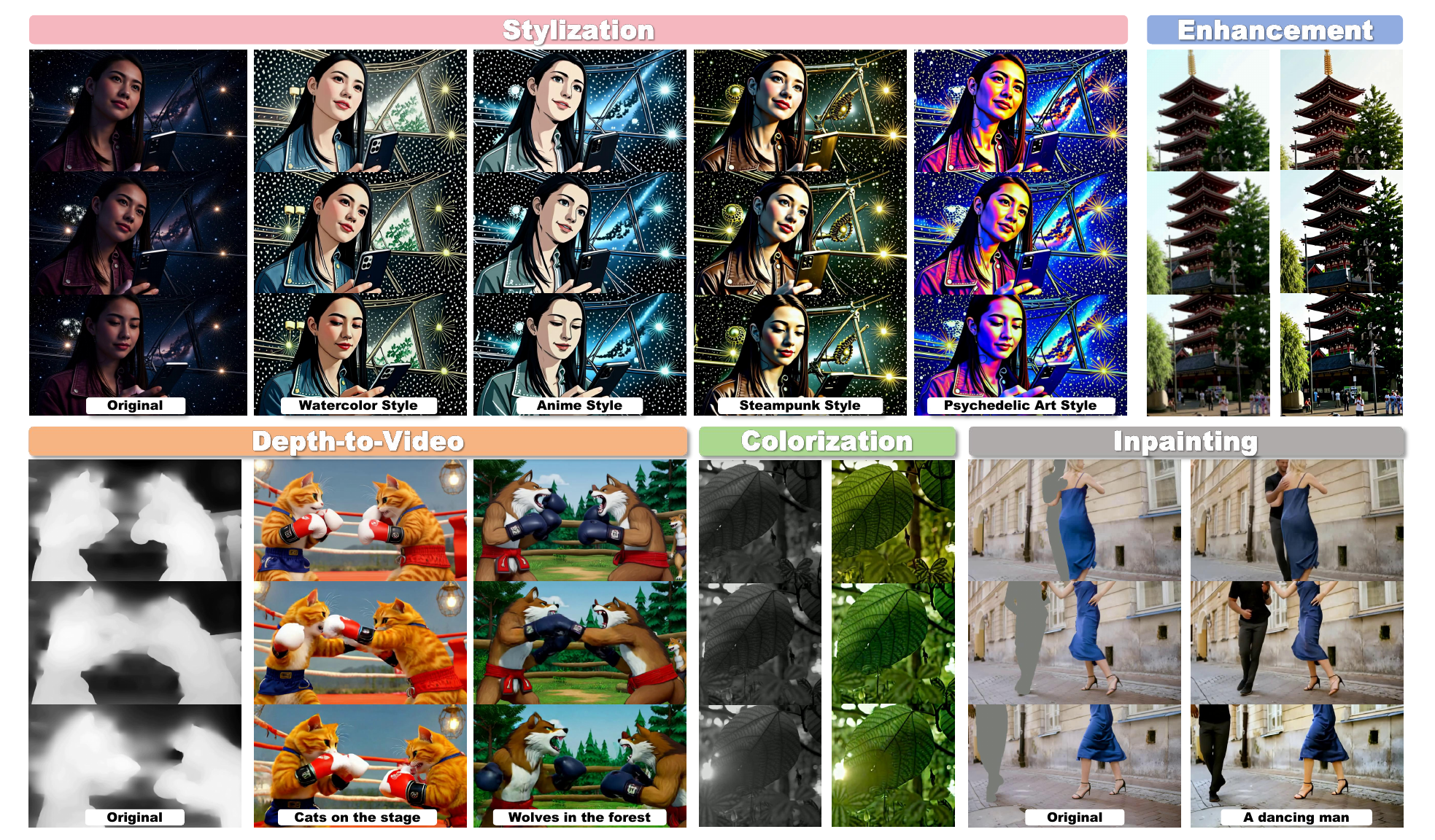}
    \caption{Visualization of the proposed SVEET. Our method can perform high-quality streaming video editing using only bidirectional model training. Here we demonstrate the results generated by SVEET on five video editing tasks, including stylization, enhancement, depth-to-video, colorization and inpainting.}
    \label{fig:teaser} 
\end{figure}

\begin{abstract}
In this paper, we propose SVEET, a framework that requires merely training on a pretrained bidirectional video diffusion model but supports high-quality streaming video editing in an auto-regressive fashion. 
To tackle this problem, we first systematically revisit existing video-to-video diffusion approaches and identify two key principles for such streaming adaptation: backbone feature disentanglement and conditional frame independence.  
Building on these insights, we develop a novel paradigm for controllable video generation. 
At its core, an auxiliary model branch encodes source video inputs with temporally independent self-attention, and the intermediate features are injected into the corresponding backbone blocks for streaming-compatible control. 
Moreover, to bridge the discrepancy between the feature spaces of bidirectional and streaming models, we propose a decoupled training scheme that explicitly enforces the orthogonality between the optimization directions of video controllability and model causality. 
Such disentanglement ensures compatibility between the two objectives at inference and facilitates smooth zero-shot knowledge transfer across heterogeneous backbone architectures. 
Extensive experiments demonstrate that SVEET achieves superior editing quality while maintaining real-time performance, attaining 15 FPS on a single H100 GPU without any auxiliary acceleration techniques. 
Codes are available \href{https://github.com/YujiaHu1109/SVEET}{here}.
\end{abstract}

\section{Introduction}
Driven by the growing demand for real-time and interactive applications, video diffusion models
\cite{blattmann2023stable,ho2022video,kong2024hunyuanvideo,liu2024sora,wan2025wan}
are increasingly moving from offline generation toward streaming settings
\cite{yang2025longlive,kodaira2025streamdiffusion,feng2025streamdiffusionv2},
where frames are generated causally to support low-latency applications such as interactive content creation and real-time avatar synthesis
\cite{mahmoud2025systematic,lu2025gas}.
Recent approaches \cite{yin2025slow,huang2025self,zhu2026causal} typically adapt pretrained bidirectional Diffusion Transformers (DiTs) into causal autoregressive models, often together with distillation for efficient sequential generation.

Despite this progress, DiT-based streaming video editing remains largely unexplored.
Existing video editing methods
\cite{kong2024hunyuanvideo,liu2024sora,wan2025wan}
rely on bidirectional full-sequence processing, which fundamentally violates streaming constraints so that making them incompatible with causal streaming applications such as live style transfer and online inpainting.
To bridge this gap, a straightforward solution is to train a streaming editing model from scratch. However, this is prohibitively expensive, as pretraining a competitive video diffusion backbone requires massive computational resources. 
According to recent works~\cite{yin2025causvid,huang2025self,zhu2026causal}, such streaming distill can take 128 H100 GPU days and involve synthesizing thousands of ODE pairs. 
Moreover, incorporating control signals further increases training cost.

Motivated by these inconveniences, in this paper, we are curious about one interesting question: \emph{Is it possible to achieve high-quality streaming video editing by training solely on a pretrained bidirectional model and directly transfer the learned editing capability to a streaming setting without any retraining?} 
To investigate this problem, we first systematically revisit existing video-to-video diffusion architectures and identify two fundamental principles for such streaming-compatible control: (1) \emph{Backbone feature disentanglement}, where the control mechanism must remain decoupled from the base model to preserve the encapsulated pretrained knowledge; and (2) \emph{Conditional frame independence}, where source video encoding must be performed in a per-frame basis without temporal dependency to ensure causal compatibility during streaming inference.

Guided by these insights, we propose \emph{SVEET}, a framework which enables \underline{s}treaming \underline{v}ideo \underline{e}diting with \underline{e}asy adap\underline{t}ion and answer the above question positively. 
At its core, SVEET introduces an auxiliary branch that encodes source video inputs using temporally independent self-attention and injects the resulting features into corresponding backbone blocks for streaming-compatible control. 
We verify that such an architecture adheres to the above guidelines and achieves plausible results in the zero-shot transfer from bidirectional to streaming settings. 

Nevertheless, a fundamental issue remains: the feature spaces of bidirectional and autoregressive models are not fully aligned, which hinders zero-shot transfer. 
Regarding this, we further propose a decoupled training scheme that explicitly enforces the orthogonality between optimization directions for video controllability and model causality. Specifically, we first use singular value decomposition (SVD)~\cite{stewart1993early,wall2003singular} to extract the dominant update directions capturing the discrepancy between autoregressive and bidirectional models and then constrain the control branch to optimize in the orthogonal subspace. 
We theoretically validate that such disentanglement facilitates the compatibility between the two objectives and enables effective zero-shot transfer across heterogeneous backbones. 

Extensive experiments across multiple editing tasks show that SVEET achieves strong editing quality while running at 15 FPS on a single H100 GPU without auxiliary acceleration.
Our contributions can be summarized as follows:
\begin{itemize}
    \item To the best of our knowledge, we are the first to study bidirectional-to-streaming video editing transfer without retraining or distilling the streaming backbone.

    \item We identify two principles for streaming-compatible transfer---backbone feature disentanglement and conditional frame independence---and develop SVEET, a temporally independent control architecture that directly transfers from a bidirectional to a streaming backbone.

    \item We introduce Orthogonal Decoupled Training (ODT) to mitigate the bidirectional-to-causal feature mismatch by decoupling controllability learning from causal dynamics. Extensive experiments demonstrate effective zero-shot transfer across diverse video editing tasks while enabling real-time streaming inference.
\end{itemize}

\section{Related Work}
\begin{figure}[t]
    \vspace{-0.3cm}
    \centering 
    \includegraphics[width=1.0\textwidth]{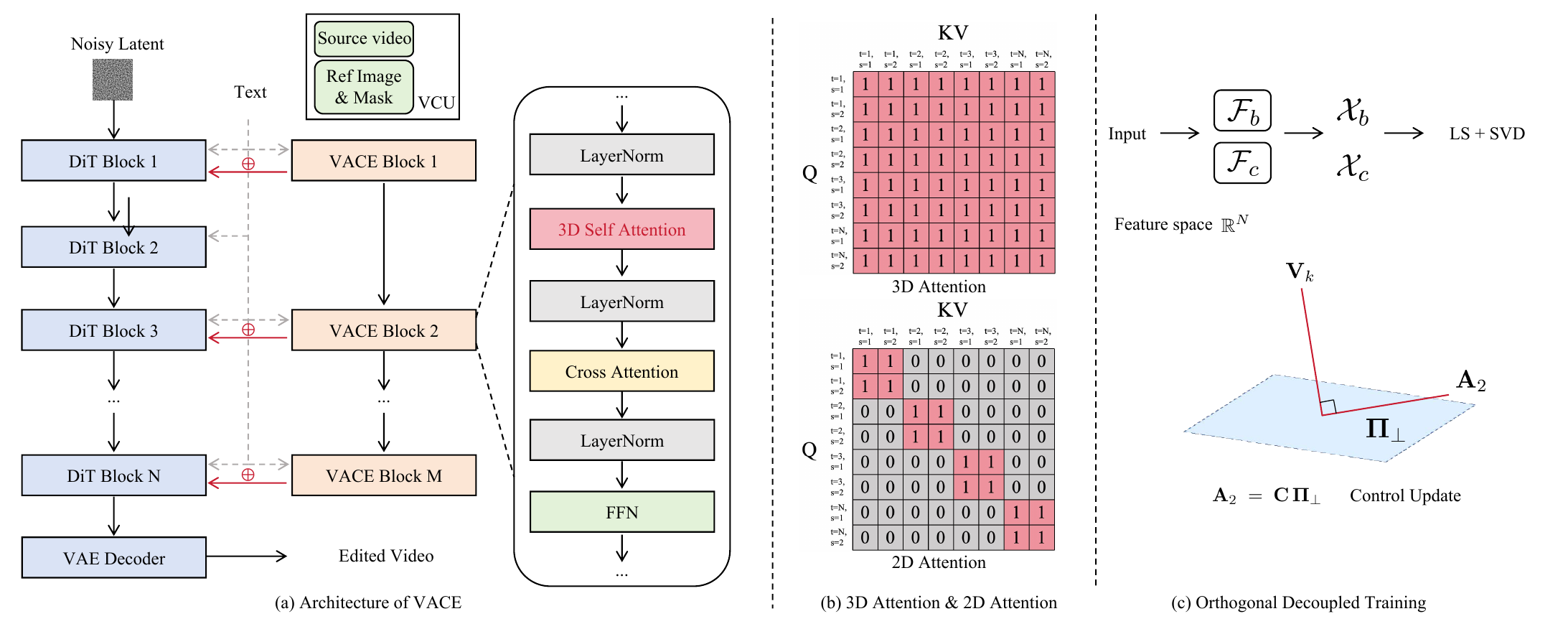}
    \vspace{-0.3cm}
    \caption{Our proposed SVEET. Our method builds upon VACE (a) and introduces a temporally independent 2D attention design in the control branch (b). To enable effective transfer from bidirectional to streaming backbones, we further incorporate an orthogonal decoupled training scheme (c) to bridge the feature discrepancy between the two settings.} 
    \label{fig:pipeline} 
    \vspace{-0.35cm}
\end{figure}

\subsection{Video Editing with Diffusion Models}
Recent video diffusion models
\cite{blattmann2023stable,ho2022video,kong2024hunyuanvideo,liu2024sora,wan2025wan}
support a broad range of editing tasks, including appearance editing
\cite{ding2023diffusionrig,bai2025uniedit},
structure-guided generation \cite{xing2024make},
and object-level manipulation \cite{ma2025magicstick}.
Early methods~\cite{lei2025animateanything,wu2023tuneavideooneshottuningimage,xu2023magicanimatetemporallyconsistenthuman,wang2024motionctrlunifiedflexiblemotion} extend image diffusion models~\cite{ho2020denoising,song2020denoising,rombach2022high} with temporal modules or attention propagation, while recent works~\cite{jiang2025vace,cheng2023consistent,bai2025scaling} increasingly adopt unified conditional architectures for flexible multi-task editing.
However, most existing approaches rely on bidirectional full-sequence processing, limiting their applicability to real-time streaming scenarios.

\subsection{Real-time Streaming Video Generation}
Recent real-time video generation methods mainly pursue two directions:
causal autoregressive generation
\cite{sun2019videobert,yan2021videogpt,singer2022make,villegas2022phenaki,henschel2025streamingt2v}
and inference acceleration through consistency or distillation
\cite{song2023consistency,luo2023latent,zheng2025large,yin2024one,salimans2022progressive,lu2025adversarial}.
Recent approaches~\cite{huang2025self,zhu2026causal,zhao2026causalforcingscalablefewstep} further combine autoregressive modeling with diffusion distillation to reduce the gap between training and inference and enable efficient streaming generation.
However, these methods primarily target video generation rather than video editing and often require costly adaptation or distillation.
In contrast, we study transferring editing capability learned on a bidirectional model directly to a pretrained streaming backbone without retraining it.

\section{Methods}
\label{gen_inst}

\subsection{Preliminaries}
\label{pre}
\noindent \textbf{Conditional Video Diffusion Architecture.}
Introducing a plug-and-play control module has become a common practice in the field of controllable generation. For instance, in video domain, Wan-VACE~\cite{jiang2025vace} extends Wan-T2V~\cite{wan2025wan} with a separate control branch for unified video editing.
As shown in Fig.~\ref{fig:pipeline}(a), the branch encodes multimodal conditions through a Video Condition Unit (VCU) and injects the resulting features additively into intermediate blocks of a frozen DiT backbone, which enables unified controllable video generation and editing without compromising the pretrained generative prior.

\noindent \textbf{Autoregressive Video Diffusion Pipeline.}
Unlike bidirectional models that jointly denoise all frames, autoregressive (AR) video diffusion models generate latent chunks
$\mathbf{z}=[\mathbf{z}_1,\ldots,\mathbf{z}_K]$ causally:
\begin{equation}
p_\theta(\mathbf{z}\mid\mathbf{c})
=
\prod_{k=1}^{K}
p_\theta(\mathbf{z}_k\mid\mathbf{z}_{<k},\mathbf{c}),
\label{eq:ar_factorization}
\end{equation}
where $\mathbf{c}$ denotes the text condition.
During inference, each chunk is denoised using the accumulated key/value cache of preceding chunks:
\begin{equation}
\mathbf{z}_k^{(t-1)} \;=\; \mathcal{F}_\theta\!\left(\mathbf{z}_k^{(t)},\, \mathbf{KV}_{<k},\, \mathbf{c},\, t\right), \qquad t = T, T{-}1, \dots, 1,
\label{eq:ar_inference}
\end{equation}
which is then updated for subsequent generation.

\begin{figure}[t]
    \vspace{-0.3cm}
    \centering 
    \includegraphics[width=1.0\textwidth]{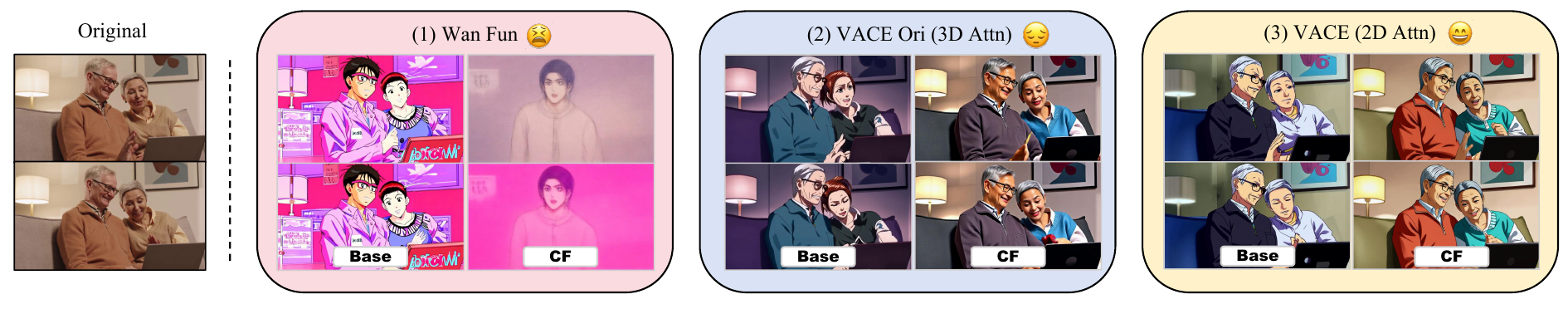}
    \vspace{-0.3cm}
    \caption{Performance of three representative structures during bidirectional-to-stream transfer.} 
    \label{fig:analysis} 
    \vspace{-0.25cm}
\end{figure}

\noindent \textbf{Bridging Bidirectional and Streaming Video Diffusion.} 
To investigate how video editing capability transfers from bidirectional diffusion models to streaming autoregressive backbones, we conduct comparative experiments on three representative architectures: (1) Wan-Fun, which fine-tunes backbone parameters during training; (2) VACE with full spatiotemporal attention; and (3) VACE with temporally independent 2D attention by transferring their learned editing capability from a bidirectional to a streaming backbone.

As shown in Fig.~\ref{fig:analysis}, VACE-Fun suffers severe degradation after transfer, almost completely failing to transfer, full spatiotemporal VACE retains only partial editing capability with unstable temporal behavior, whereas the 2D-attention variant transfers substantially more reliably.

These observations suggest two key principles for bridging bidirectional and streaming video diffusion:
(P1) Backbone feature disentanglement: control learning should remain decoupled from the bidirectional backbone;
and
(P2) Conditional frame independence: source-video conditioning should avoid cross-frame dependencies that conflict with causal chunk-wise inference.
These principles motivate the design of SVEET.

\subsection{SVEET}

\noindent \textbf{Problem Statement.}
We consider two pretrained video diffusion backbones: a bidirectional model $\mathcal{F}_b$, which jointly denoises all video frames, and a causal streaming model $\mathcal{F}_c$, which generates videos sequentially in a chunk-wise autoregressive manner as introduced in Sec.~\ref{pre}.
Our goal is to learn video editing capability on the frozen $\mathcal{F}_b$
and directly transfer it to the frozen $\mathcal{F}_c$ without retraining or
fine-tuning the streaming backbone.

This transfer presents two challenges: the conditioning pathway must remain compatible with causal inference, while the feature spaces of the two backbones are not perfectly aligned.
Accordingly, SVEET consists of two complementary components:
(i) a temporally independent control branch for causal-compatible conditioning,
and (ii) Orthogonal Decoupled Training (ODT) for mitigating the
bidirectional-to-causal feature discrepancy.

\noindent \textbf{Temporally Independent Control.} 
\label{Temporally}
We build the control pathway upon a VACE-like conditional video diffusion backbone, which already satisfies backbone feature disentanglement. Motivated by the conditional frame independence principle identified in our analysis, we further modify the control branch by replacing the original full spatiotemporal self-attention with temporally independent 2D spatial attention to ensure frame-wise independent encoding, as shown in Fig.~\ref{fig:pipeline}(b).
Concretely, denoting each token by its frame index $t$ and spatial index $s$, the attention mask is defined as
\begin{equation}
M_{(t,s),(t',s')} \;=\;
\begin{cases}
1, & \text{if } t = t'; \\
0, & \text{otherwise},
\end{cases}
\label{eq:frame_mask}
\end{equation}
such that each query attends only to tokens from the same source frame, reducing the attention to a per-frame 2D spatial self-attention. Consequently, the conditioning representation of frame $t$ depends only on that frame, while temporal coherence of the generated video remains modeled by
the causal streaming backbone through its autoregressive history. As a direct consequence, the control branch requires no key/value cache during streaming inference, avoiding introducing any additional cache-growth overhead on top of the AR backbone in Eq.~\ref{eq:ar_inference}.

\noindent \textbf{Orthogonal Decoupled Training.} Although the control branch is trained on a bidirectional diffusion backbone $\mathcal{F}_b$, it is deployed on a causal streaming backbone $\mathcal{F}_c$ with different parameters and feature distributions. This mismatch leads to degraded transfer of learned control signals when directly applied in streaming settings. We address this issue by explicitly decoupling controllability learning from the bidirectional-to-causal feature discrepancy.

We first probe the feature gap in a data-driven way. Given a small calibration set of video--prompt pairs, we feed identical inputs to $\mathcal{F}_b$ and $\mathcal{F}_c$ and record the per-block hidden states $\mathbf{X}_b^{(\ell)}, \mathbf{X}_c^{(\ell)} \in \mathbb{R}^{n \times N}$, where $N$ is the total number of cached tokens at block $\ell$. A closed-form linear map between the two feature streams is obtained via ridge regression, which corresponds to a regularized least-squares estimation:

\begin{equation}
\mathbf{W}_1^{(\ell)} =
\arg\min_{\mathbf{W}} \left\| \mathbf{W}\mathbf{X}_b^{(\ell)} - \mathbf{X}_c^{(\ell)} \right\|_F^2
+ \lambda \|\mathbf{W}\|_F^2,
\label{eq:lsq}
\end{equation}
where $\lambda$ is the ridge regularization coefficient. Taking the bidirectional feature itself as the reference, i.e., $\mathbf{W}_b = \mathbf{I}$, the residual transformation:
\begin{equation}
\mathbf{A}_1^{(\ell)} \;=\; \mathbf{W}_1^{(\ell)} - \mathbf{I} \;\in\; \mathbb{R}^{n\times n}
\label{eq:a1}
\end{equation}
captures, in feature space, the \emph{update direction} that turns a bidirectional representation into its causal counterpart at block $\ell$.
To localize the directions along which $\mathbf{A}_1^{(\ell)}$ acts most strongly, we apply a thin singular value decomposition
\begin{equation}
\mathbf{A}_1^{(\ell)} \;=\; \mathbf{U}^{(\ell)}\,\boldsymbol{\Sigma}^{(\ell)}\,\mathbf{V}^{(\ell)\top},
\label{eq:svd}
\end{equation}
and retain the top-$k$ right singular vectors $\mathbf{V}_k^{(\ell)} \in \mathbb{R}^{n\times k}$ associated with the $k$ largest singular values. 
The columns of $\mathbf{V}_k^{(\ell)}$ form an orthonormal basis of the \emph{discrepancy subspace}, capturing the dominant directions of the bidirectional-to-causal feature shift. We empirically verify that the discrepancy energy is concentrated in a compact set of singular directions across DiT layers in Appendix~\ref{appendix:feature_analysis}.
The orthogonal projector onto its complement is
\begin{equation}
\boldsymbol{\Pi}_{\perp}^{(\ell)} \;=\; \mathbf{I} - \mathbf{V}_k^{(\ell)}\mathbf{V}_k^{(\ell)\top},
\label{eq:proj}
\end{equation}
which satisfies $\boldsymbol{\Pi}_{\perp}^{(\ell)}\mathbf{V}_k^{(\ell)} = \mathbf{0}$ by construction. 

At each block $\ell$, the control branch introduces a learnable update
$\mathbf{A}_2^{(\ell)}$ to the backbone feature transformation.
To prevent this update from relying on the dominant
bidirectional-to-causal discrepancy directions, we constrain it as
\begin{equation}
\mathbf{A}_2^{(\ell)}
=
\mathbf{C}^{(\ell)}
\boldsymbol{\Pi}_{\perp}^{(\ell)},
\qquad
\mathbf{W}_2^{(\ell)}
=
\mathbf{W}_b^{(\ell)}
+
\mathbf{A}_2^{(\ell)},
\label{eq:reparam}
\end{equation}
where $\mathbf{C}^{(\ell)}$ denotes the unconstrained trainable update, which is parameterized in practice by LoRA adapters~\cite{hu2021loralowrankadaptationlarge}, while $\boldsymbol{\Pi}_{\perp}^{(\ell)}$ is precomputed and frozen throughout training.

This design constrains the effective control update to the orthogonal
complement of the dominant discrepancy subspace, thereby reducing its overlap
with the principal feature directions associated with the
bidirectional-to-causal transition, as illustrated in
Fig.~\ref{fig:pipeline}(c).
Please refer to the next section for the theoretical insights behind this
strategy.

\section{Theoretical Analysis}\label{sec:theory}

In this section, we provide theoretical insights for the proposed orthogonal decoupled training strategy. 
Our analysis takes inspiration from recent study on model merging~\cite{cheng2025whoever} and composable diffusion models~\cite{liu2022compositional}. 

\begin{theorem}[Informal Version: Feature-Level Decoupling]
Let \(\Delta W_1\) and \(\Delta W_2\) denote two functionality-specific parameter updates. Suppose the dominant singular subspace of \(\Delta W_1\) is removed from \(\Delta W_2\) via
\begin{equation}
\widetilde{\Delta W}_2
=
\Delta W_2(I-P_k),
\end{equation}
where \(P_k\) projects onto the top singular directions of \(\Delta W_1\).

Then the interference induced by the second functionality on the first functionality is upper bounded by the residual tail energy outside the dominant singular subspace of \(\Delta W_1\). In particular, if the first functionality is approximately low-rank, the induced interference becomes negligible.

Meanwhile, the capacity reduction of the second functionality depends only on its overlap with the removed singular subspace.

\end{theorem}


\begin{theorem}[Informal Version: Orthogonality Improves Output Additivity]
Consider a nonlinear model \(f(W)\) and two functionality-specific updates \(\Delta W_1\) and \(\widetilde{\Delta W}_2\). The deviation from ideal additive composition is measured by
\begin{equation}
f(W+\Delta W_1+\widetilde{\Delta W}_2)
-
f(W+\Delta W_1)
-
f(W+\widetilde{\Delta W}_2)
+
f(W).
\end{equation}

If the dominant singular subspace of \(\Delta W_1\) is removed from \(\Delta W_2\), then the higher-order interaction between the two functionalities becomes bounded by the residual tail singular values of \(\Delta W_1\).

Consequently, orthogonal subspace projection suppresses nonlinear coupling between functionalities and improves output compositionality.

\end{theorem}

\noindent
The formal statement and proof are provided in
Appendix~\ref{appendix:theory}.

Intuitively, the above analysis suggests that enforcing orthogonality between the optimization directions of video controllability and model causality preserves the feature components required by both objectives, thereby yielding a bounded output error due to the inherent compositionality of diffusion model outputs. 

\section{Experiments}
\label{headings}

\subsection{Experimental Settings}
\noindent \textbf{Training Settings.} 
We implement SVEET with Wan2.1-1.3B-VACE~\cite{jiang2025vace} as the bidirectional backbone and chunk-wise Causal Forcing~\cite{zhu2026causal} as the default streaming backbone.
The VACE control branch is optimized using LoRA adapters~\cite{hu2021loralowrankadaptationlarge} with rank 128, with all pretrained backbone parameters kept frozen. 
We adopt the AdamW optimizer \cite{loshchilov2017decoupled} with a learning rate of $1\times10^{-4}$ and weight decay of $10^{-2}$. 
For ODT, the layer-wise projector $\Pi_\perp$ retains $80\%$ of the cumulative spectral energy at each DiT block.
Each task is trained for 10 epochs on a single NVIDIA A100 (80GB) with batch size 1, using 81-frame video clips at $480\times832$ resolution.
We mainly evaluate three representative editing tasks: style transfer, video inpainting, and depth-to-video generation. Further training details are provided in Appendix~\ref{appendix:training}.

\noindent \textbf{Dataset Setup.} 
For training, we construct task-specific datasets by sampling videos from existing datasets. Specifically, style transfer data is sampled from Ditto \cite{bai2025scaling}, while videos for inpainting and depth-to-video generation are sampled from VPData \cite{bian2025videopainter}. For the depth-to-video task, depth conditions are further extracted using Video-Depth-Anything \cite{chen2025video}. Each task contains approximately 6K--13K training video samples.
For evaluation, we construct held-out test sets for each task. Specifically, we randomly sample 120 video pairs from Ditto \cite{bai2025scaling} for style transfer, and 80 samples from VPData \cite{bian2025videopainter} for both inpainting and depth-to-video tasks. For depth-to-video evaluation, depth maps are generated using the same depth estimation pipeline.
More dataset details are provided in Appendix~\ref{appendix:dataset}.

\noindent \textbf{Baselines.}  
We compare our method against two categories of baselines. For existing approaches, we consider (i) SDEdit+CF, a training-free baseline that adapts SDEdit~\cite{meng2021sdedit} to Causal Forcing \cite{zhu2026causal}, (ii) StreamDiffusionV2 ~\cite{feng2025streamdiffusionv2}, (iii) Daydream+CF~\cite{fosdick2026adaptingvacerealtimeautoregressive}, (iv) LiveEdit~\cite{wang2026liveedit}. 
As open-source real-time video editing models remain scarce, we further construct two controlled baselines that share our overall pipeline but differ in key design choices: (v) Full-Attn, which retains full spatiotemporal attention at control branch, and (vi) Channel-Concat Control, which replaces the control branch with channel-wise concatenation of source and noisy latents.

\begin{table*}[t]
\vspace{-0.1cm}
\centering
\fontsize{8pt}{8.4pt}\selectfont
\renewcommand\arraystretch{1.05}
\setlength{\tabcolsep}{0.45pt}  

\caption{Quantitative results on three video editing tasks.}
\label{tab:full_quantitative_results}

\begin{subtable}{\textwidth}
\centering
\caption{Video Style Transfer}
\vspace{-0.15cm}
\label{tab:style_transfer}
\begin{tabular}{l|c|c|cccccc}
\toprule
\multirow{2}{*}{Method} & VLM & Text Alignment & \multicolumn{6}{c}{VBench Evaluation} \\
\cmidrule(lr){2-2} \cmidrule(lr){3-3} \cmidrule(lr){4-9}

& \makecell{Editing\\Accuracy $\uparrow$}
& \makecell{CLIP\\$\uparrow$}
& \makecell{Subject\\Consistency $\uparrow$}
& \makecell{Background\\Consistency $\uparrow$}
& \makecell{Temporal\\Flickering $\uparrow$}
& \makecell{Motion\\Smoothness $\uparrow$}
& \makecell{Aesthetic\\Quality $\uparrow$} 
& \makecell{Overall\\Consistency $\uparrow$}\\
\midrule
SDEdit+CF & 3.1328 & 0.1412 & 0.8674 & 0.8949 & 0.9734 & 0.9840 & 0.5134 & 0.0916 \\
SDV2 & 4.2297 & 0.1833 & 0.9011 & \underline{0.9238} & \textbf{0.9857} & 0.9798 & 0.4650 & 0.1055 \\
DayDream+CF      & 5.2321 & 0.2276 & 0.9147 & 0.9203 & 0.9775 & 0.9882 & 0.5049 & 0.0872 \\
LiveEdit       & 5.8672 & \underline{0.2279} & 0.9276 & 0.9219 & \underline{0.9832} & \underline{0.9892} & 0.5289 & 0.1098 \\
\midrule
3D VACE attn       & \underline{7.0653} & 0.2105 & \underline{0.9313} & 0.9175 & 0.9592 & 0.9809 & \underline{0.5335} & 0.1046 \\
Channel concat        & 1.8902 & 0.1414 & 0.7259 & 0.8848 & 0.9703 & 0.9763 & 0.3655 & \textbf{0.1229} \\
\midrule
Ours            & \textbf{7.4322} & \textbf{0.2287} & \textbf{0.9447} & \textbf{0.9298} & 0.9819 & \textbf{0.9898} & \textbf{0.5550} & \underline{0.1123} \\
\bottomrule
\end{tabular}
\end{subtable}


\begin{subtable}{\textwidth}
\centering
\vspace{+0.1cm}
\caption{Video Inpainting}
\vspace{-0.1cm}
\label{tab:content_editing}
\begin{tabular}{l|c|c|cccccc}
\toprule
\multirow{2}{*}{Method} & VLM & Text Alignment & \multicolumn{6}{c}{VBench Evaluation} \\
\cmidrule(lr){2-2} \cmidrule(lr){3-3} \cmidrule(lr){4-9}

& \makecell{Editing\\Accuracy $\uparrow$}
& \makecell{CLIP\\$\uparrow$}
& \makecell{Subject\\Consistency $\uparrow$}
& \makecell{Background\\Consistency $\uparrow$}
& \makecell{Temporal\\Flickering $\uparrow$}
& \makecell{Motion\\Smoothness $\uparrow$}
& \makecell{Aesthetic\\Quality $\uparrow$} 
& \makecell{Overall\\Consistency $\uparrow$}\\
\midrule
SDEdit+CF & 6.7137 & 0.2816 & 0.9422 & 0.9363 & 0.9863 & 0.9909 & \underline{0.5496} & 0.0928 \\
SDV2 & 6.8580 & 0.2654 & 0.9609 & \textbf{0.9551} & 0.9836 & \textbf{0.9945} & 0.4867 & \underline{0.0964} \\
DayDream+CF      & 7.0115 & 0.2807 & 0.9579 & 0.9420 & \underline{0.9880} & 0.9935 & 0.5434 & 0.0955 \\
LiveEdit        & 7.7833 & 0.2803 & \underline{0.9632} & 0.9476 & 0.9873 & 0.9933 & \textbf{0.5498} & 0.0962 \\
\midrule
3D VACE attn       & \underline{8.0583} & \underline{0.2882} & 0.9554 & 0.9426 & 0.9866 & 0.9933 & 0.5252 & 0.0947 \\
Channel concat        & 6.8913 & 0.2690 & 0.9475 & 0.9369 & 0.9840 & 0.9933 & 0.4935 & 0.0926 \\
\midrule
Ours            & \textbf{8.4667} & \textbf{0.2895} & \textbf{0.9653} & \underline{0.9478} & \textbf{0.9886} & \underline{0.9936} & 0.5297 & \textbf{0.1014} \\
\bottomrule
\end{tabular}
\end{subtable}


\begin{subtable}{\textwidth}
\centering
\vspace{+0.1cm}
\caption{Depth-to-Video}
\vspace{-0.1cm}
\label{tab:video_generation}
\begin{tabular}{l|c|c|cccccc}
\toprule
\multirow{2}{*}{Method} & VLM & Text Alignment & \multicolumn{6}{c}{VBench Evaluation} \\
\cmidrule(lr){2-2} \cmidrule(lr){3-3} \cmidrule(lr){4-9}

& \makecell{Editing\\Accuracy $\uparrow$}
& \makecell{CLIP\\$\uparrow$}
& \makecell{Subject\\Consistency $\uparrow$}
& \makecell{Background\\Consistency $\uparrow$}
& \makecell{Temporal\\Flickering $\uparrow$}
& \makecell{Motion\\Smoothness $\uparrow$}
& \makecell{Aesthetic\\Quality $\uparrow$} 
& \makecell{Overall\\Consistency $\uparrow$}\\
\midrule
SDEdit+CF & 4.3673 & 0.2658 & 0.9239 & 0.9208 & 0.9920 & \underline{0.9943} & 0.4854 & \underline{0.0929} \\
SDV2 & 5.0505 & 0.2275 & 0.9411 & 0.9513 & \textbf{0.9961} & \textbf{0.9961} & 0.4606 & \textbf{0.1030} \\
DayDream+CF      & \underline{7.1502} & 0.2892 & 0.9332 & 0.9239 & 0.9818 & 0.9884 & 0.5096 & 0.0811 \\
LiveEdit        & 5.1892 & 0.2896 & 0.9613 & 0.9507 & \underline{0.9958} & 0.9842 & 0.5273 & 0.0809 \\
\midrule
3D VACE attn       & 5.7484 & \underline{0.2907} & \underline{0.9720} & \underline{0.9517} & 0.9855 & 0.9919 & \underline{0.5466} & 0.0793 \\
Channel concat        & 6.0830 & 0.2892 & 0.9332 & 0.9239 & 0.9818 & 0.9884 & 0.5096 & 0.0811 \\
\midrule
Ours            & \textbf{8.4083} & \textbf{0.2928} & \textbf{0.9727} & \textbf{0.9521} & 0.9892 & 0.9882 & \textbf{0.5511} & 0.0813 \\
\bottomrule
\end{tabular}
\end{subtable}
\vspace{-0.4cm} 

\end{table*}

\noindent \textbf{Evaluation Metrics.}
We evaluate edited videos mainly from three perspectives. (1) VLM Assessment: we use GPT-4o~\cite{hurst2024gpt} to assess editing quality, including faithfulness to the instruction and visual coherence, and report average scores on a 1–10 Likert scale; (2) Text alignment: we assess this using CLIP-T~\cite{radford2021learning}, computed as the cosine similarity between the editing prompt and sampled video frames encoded by CLIP ViT-L/14; (3) VBench Evaluation: we evaluate video quality with six VBench~\cite{huang2024vbench} metrics relevant to our tasks including subject consistency, background consistency, temporal flickering, motion smoothness, aesthetic quality and overall consistency. In addition, we conduct a user study involving 10  participants, who are asked to score 15 generated videos in terms of editing correctness, structural preservation, and motion smoothness.
\subsection{Main Results}
\noindent \textbf{Quantitative Comparison.}
Tab.~\ref{tab:full_quantitative_results} presents quantitative comparisons across the three editing tasks on our test sets. 
Overall, our method achieves the best average performance across most metrics and tasks, with particularly strong gains in VLM scores, where it consistently outperforms all baselines by a clear margin. Tab.~\ref{tab:user} further reports the results of our human evaluation, showing that videos generated by ours are consistently preferred by participants over existing methods.

\begin{figure}[t]
    \vspace{-0.3cm}
    \centering 

    \includegraphics[width=1.0\textwidth]{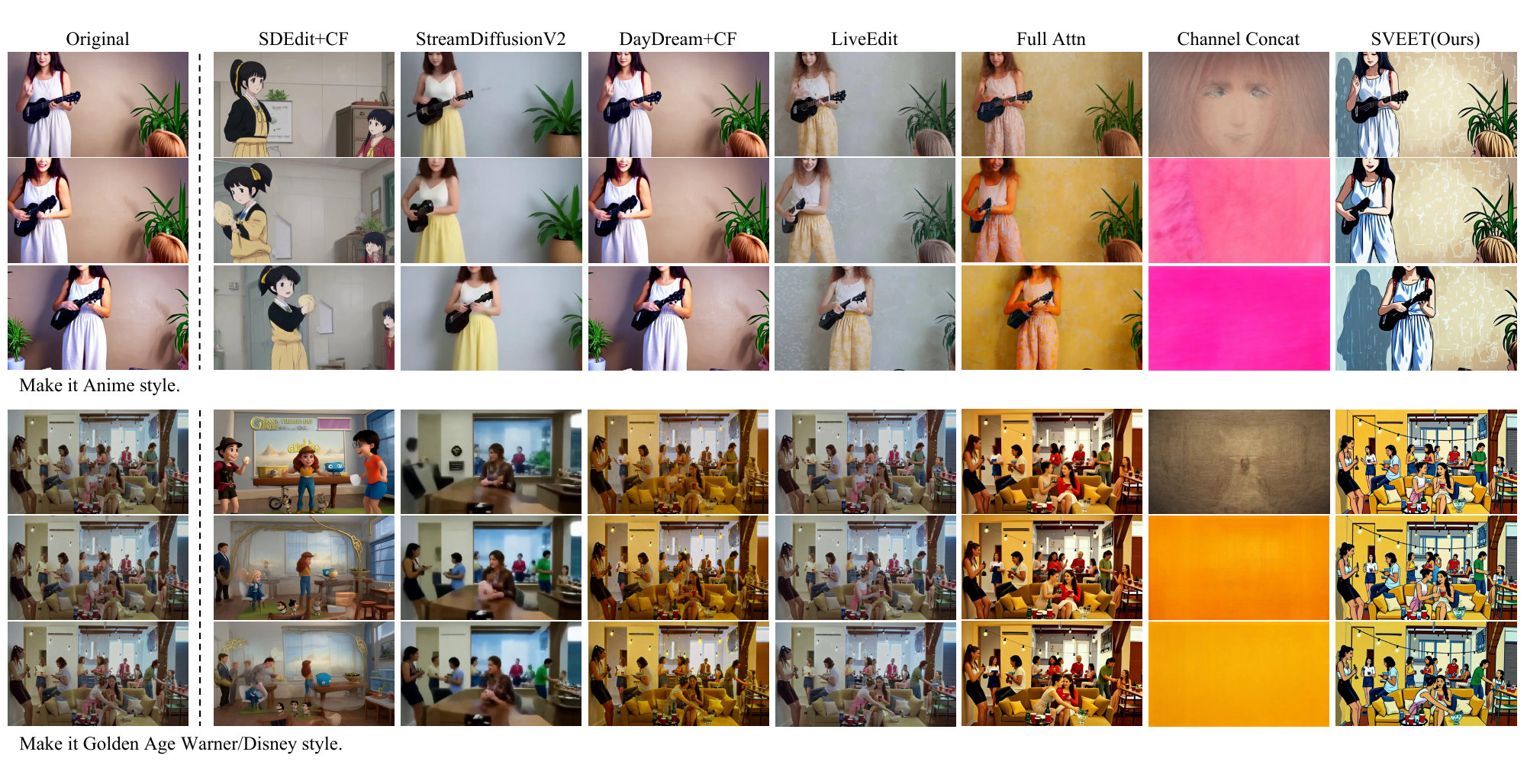}
    \vspace{-0.5cm}
    \caption{Qualitative comparison results on stylization.}
    \label{fig:stylization}

    \vspace{0.1cm}

    \includegraphics[width=1.0\textwidth]{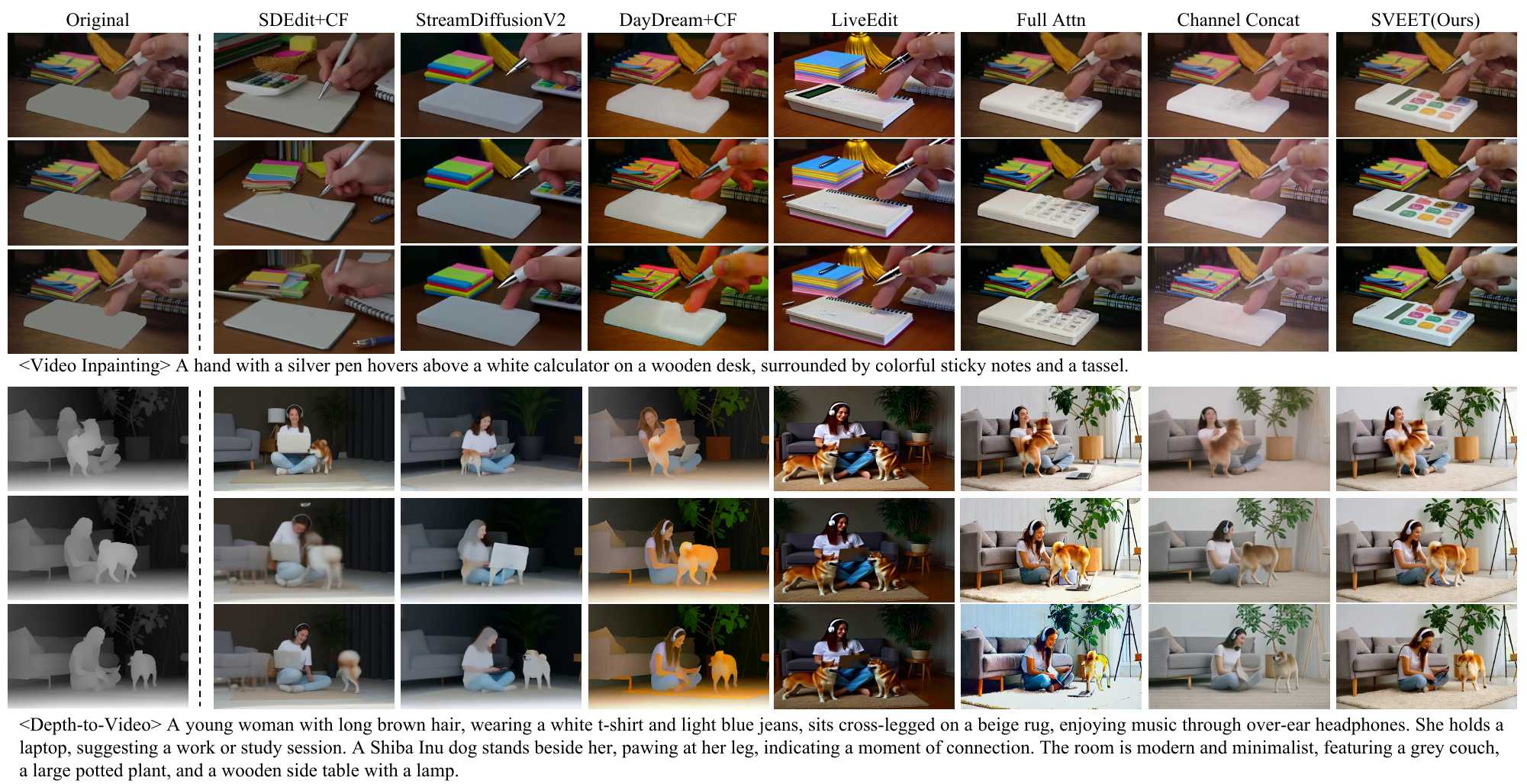}
    \vspace{-0.3cm}
    \caption{Qualitative comparison results on inpainting and depth-to-video generation.}
    \label{fig:compare}
    
    \vspace{-0.15cm}
\end{figure}

\noindent \textbf{Qualitative Comparison.}
We present qualitative results in Fig.~\ref{fig:stylization} and Fig.~\ref{fig:compare}. As shown in Fig.~\ref{fig:stylization}, our method achieves the best overall performance on the stylization task, consistently preserving source structure, faithfully transferring the target style, and maintaining strong temporal consistency. In contrast, SDEdit+CF~\cite{meng2021sdedit} and StreamDiffusionV2~\cite{feng2025streamdiffusionv2} significantly alter the original video content, while Channel Concat often fails to produce meaningful results. Daydream+CF and LiveEdit~\cite{fosdick2026adaptingvacerealtimeautoregressive} struggle to generate the target style, and Full-Attn baseline shows weaker style intensity and inferior local detail alignment compared to our method.
As shown in Fig.~\ref{fig:compare}, on inpainting and depth-to-video tasks, our method also produces higher-quality outputs with noticeably better temporal stability. Baselines such as SDEdit+CF, StreamDiffusionV2, and Daydream+CF exhibit weaker generation quality, while LiveEdit, Channel Concat and Full-Attn variants consistently underperform in terms of visual fidelity and overall consistency. Additional transfer results on other streaming backbones are provided in Appendix~\ref{appendix:results2}.

\subsection{Ablation Studies}
\noindent \textbf{Ablations on transfer strategies.}
We first compare four controlled variants to disentangle the effects of
temporal attention design and ODT: 3D/2D attention with or without ODT.

\noindent
\begin{minipage}[t]{0.57\textwidth}
\vspace{-0.3cm}
As shown in Tab.~\ref{tab:ablation} and Fig.~\ref{fig:ablation_vis},
2D attention improves source preservation and streaming stability, while ODT
further reduces transfer-induced appearance drift.
Their combination achieves the best overall trade-off, supporting the
complementarity of temporal independence and feature-space decoupling.
We also compare two alternative transfer strategies. Inference-stage projection where post-hoc projection using the estimated mapping matrix causes substantial structural drift, while two-stage teacher forcing improves stability but still degrades local content and requires additional adaptation of the streaming backbone.

\end{minipage}
\hfill
\begin{minipage}[t]{0.40\textwidth}
\vspace{-0.2cm}
\centering
\footnotesize
\renewcommand{\arraystretch}{1.05}
\setlength{\tabcolsep}{3pt}

\captionof{table}{Quantitative ablation results.}
\label{tab:ablation}

\begin{tabular}{lcc}
\toprule
Settings &
\makecell{VLM \\ $\uparrow$} &
\makecell{Motion \\ Smoothness $\uparrow$} \\
\midrule
3D w/o ODT & 7.0653 & 0.9809 \\
3D w/ ODT  & 7.3315 & 0.9856 \\
2D w/o ODT & 7.1898 & 0.9830 \\
2D w/ ODT  & \textbf{7.4322} & \textbf{0.9898} \\
\midrule
Infer-stage Proj & 6.4250 & 0.9848 \\
Two-stage TF     & 7.0927 & 0.9879 \\
\bottomrule
\end{tabular}

\end{minipage}

\begin{figure}[H]
    \centering
    \includegraphics[width=\linewidth]{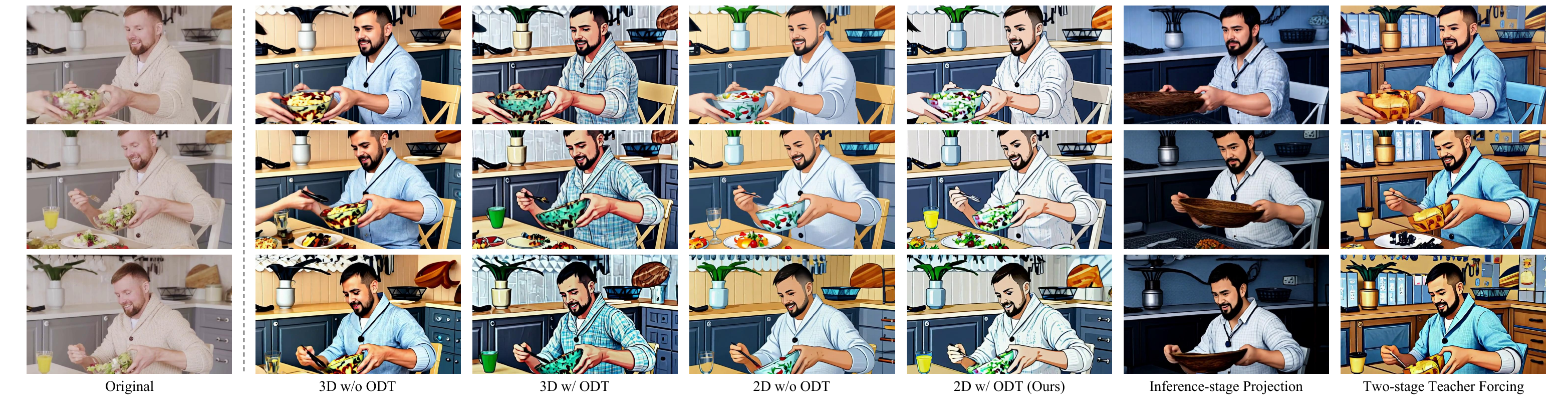}
    \caption{Visual comparison of different transfer strategies.}
    \label{fig:ablation_vis}
    \vspace{-0.2cm}
\end{figure}

\begin{figure}[H]
\centering
\begin{minipage}[ht]{0.58\textwidth}
    \centering
    \small
    \renewcommand\arraystretch{1}
    \setlength{\tabcolsep}{0.5pt}
    \vspace{-0.1cm}
    \captionof{table}{User study.}
    \label{tab:user}
    \begin{tabular}{lccc}
    \toprule
    Method & \makecell{Editing \\ Correctness $\uparrow$} & \makecell{Structural \\ Preservation $\uparrow$} & \makecell{Overall \\ Smoothness $\uparrow$} \\
    \midrule
    SDEdit+CF & 5.3867 & 2.8133 & 4.8467 \\
    SDV2 & 5.3267 & 2.9133 & 6.5800 \\
    DayDream+CF & 7.8533 & 8.6200 & 7.9533 \\
    \midrule
    Ours & \textbf{9.0800} & \textbf{9.2067} & \textbf{8.1133} \\
    \bottomrule
    \end{tabular}
    \vspace{-0.1cm}

\end{minipage}
\vspace{-0.2cm}
\hfill
\begin{minipage}[ht]{0.4\textwidth}
    \centering
    \vspace{-0.2cm}
    \includegraphics[width=\linewidth]{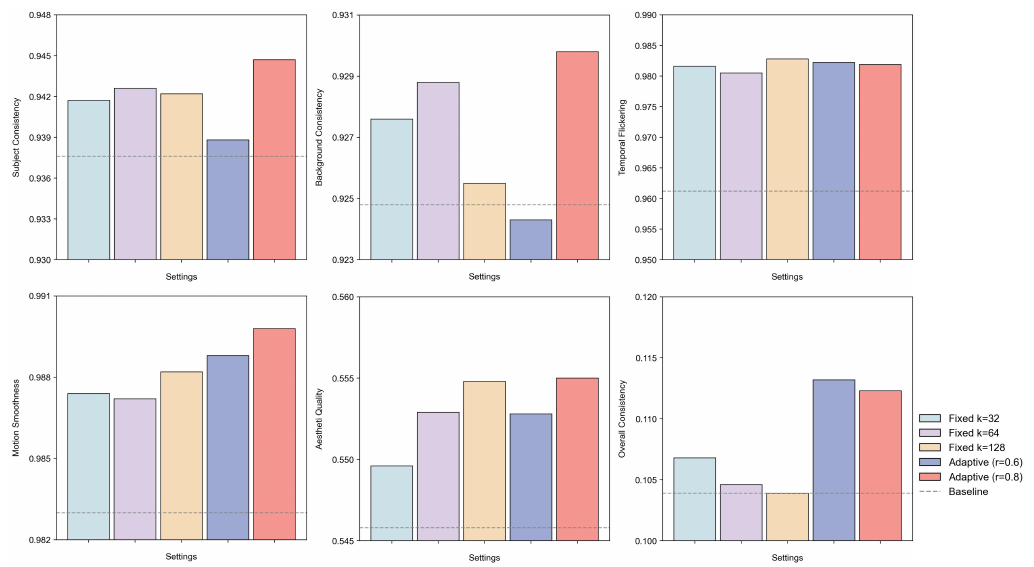}
    \vspace{-0.53cm}
    \caption{Effect of different rank choice.}
    \label{fig:k}
\end{minipage}
\vspace{-0.1cm}
\end{figure}
\noindent \textbf{Choice of subspace rank.}
We further study the effect of the subspace rank $k$ used in the orthogonal projection. Since the discrepancy between bidirectional and streaming backbones varies across different DiT layers, using a fixed rank for all layers is suboptimal. Therefore, we adopt an adaptive strategy that determines $k$ according to cumulative spectral energy. Here we compare fixed-rank settings ($k = 32, 64, 128$) with adaptive variants that retain a fixed proportion of energy ($r = 60\%, 80 \%$).
As shown in Fig. \ref{fig:k}, fixed-rank settings generally underperform adaptive strategies, with small fixed ranks yield only marginal improvements over direct transfer. In contrast, adaptive rank selection leads to stronger overall gains, with $r=80\%$ achieving comparatively better performance across most metrics.

\section{Conclusions}
\label{others}
We present SVEET, a high-quality streaming video editing framework trained solely on a pretrained bidirectional video diffusion model. We identify the key principles enabling bidirectional-to-streaming transfer without backbone retraining, and based on these insights, propose a temporally independent control branch together with an orthogonal decoupled training scheme to bridge the feature gap between bidirectional and autoregressive models for zero-shot streaming transfer. Extensive experiments demonstrate that SVEET achieves strong editing quality while maintaining real-time performance, providing a simple and effective paradigm for bridging offline and streaming video generation without costly retraining or distillation. In terms of limitations, our method still relies on the capability of the underlying bidirectional editing model, and extending this transfer paradigm to broader editing tasks and more heterogeneous backbones remains an direction for future work.

\bibliography{iclr2027_conference}
\bibliographystyle{iclr2027_conference}

\appendix
\clearpage

\section{Empirical Verification of the Assumption}
\label{appendix:feature_analysis}

We further examine the feature-space assumptions underlying ODT.
Fig.~\ref{fig:feature_analysis} shows the cumulative singular-value energy of the bidirectional-to-causal residual transformation
$\mathbf{A}_1^{(\ell)}$ at representative DiT blocks.
The spectra exhibit clear concentration: only a relatively small fraction of singular directions is required to capture most of the discrepancy energy.
As summarized in Fig.~\ref{fig:feature_analysis}(b), approximately
$4\%$--$21\%$ of the directions are sufficient to explain $80\%$ of the spectral energy across different layers.
This observation supports our modeling of the bidirectional-to-causal shift with a compact dominant discrepancy subspace, while also motivating the layer-adaptive rank selection used in ODT.

\begin{figure}[H]
    \centering
    \includegraphics[width=\linewidth]{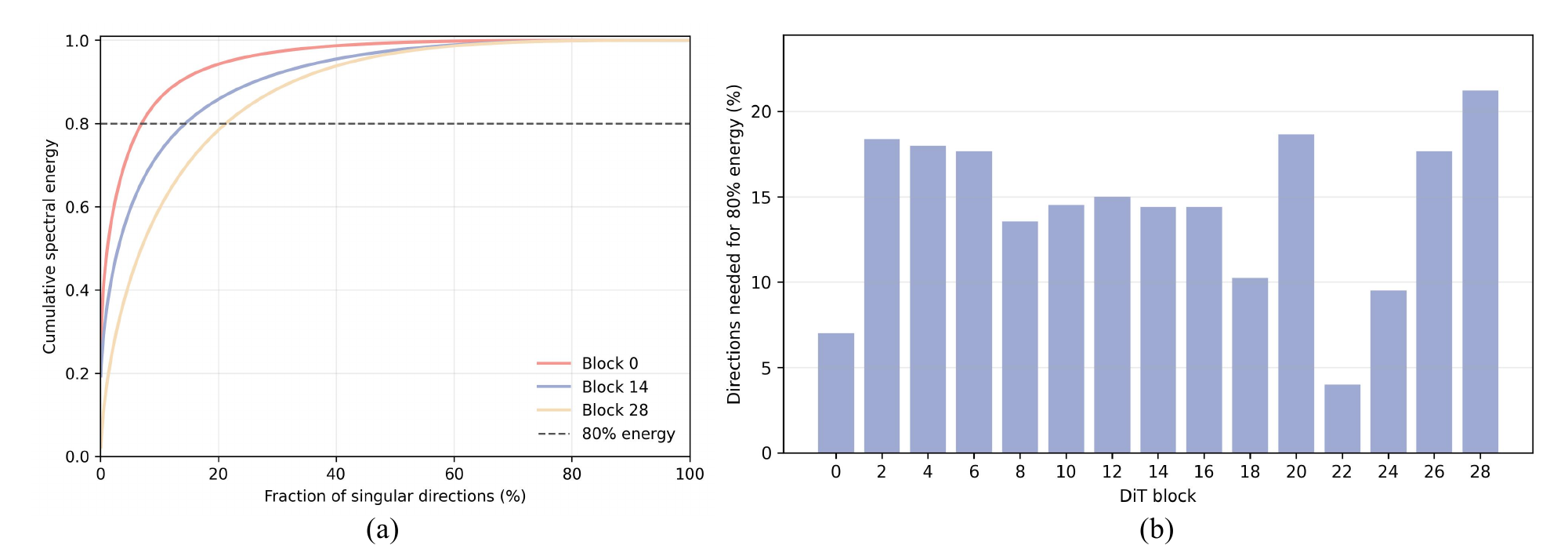}
    \caption{Empirical verification of the assumption.}
    \label{fig:feature_analysis}
    \vspace{-0.2cm}
\end{figure}

\setcounter{theorem}{0}

\section{Theoretical Derivation}\label{appendix:theory}
We present the full theoretical derivations corresponding to Sec.~\ref{sec:theory}, including the problem setup, assumptions, and proofs, to establish the validity of the proposed orthogonal decoupled training method. 

\paragraph{Setup.}
Consider a linear transformation
\begin{equation}
Y = XW,
\end{equation}
where \(X\in\mathbb{R}^{n\times d}\) denotes the input feature matrix and \(W\in\mathbb{R}^{d\times m}\) denotes the model parameter.

Suppose two functionalities are represented by two parameter updates
\begin{equation}
\Delta W_1,\quad \Delta W_2.
\end{equation}

To reduce interference between the two functionalities, we compute the truncated SVD of \(\Delta W_1\):
\begin{equation}
\Delta W_1 = U\Sigma V^\top,
\end{equation}
and define \(V_k\in\mathbb{R}^{d\times k}\) as the matrix containing the top-\(k\) right singular vectors. We further define the orthogonal projector
\begin{equation}
P_k = V_kV_k^\top.
\end{equation}

The second update is constrained to lie in the orthogonal complement of the dominant subspace of \(\Delta W_1\):
\begin{equation}
\widetilde{\Delta W}_2
=
\Delta W_2(I-P_k).
\end{equation}

\paragraph{Assumption.}
Assume there exists a projection matrix \(P\) such that the first functionality depends only on the subspace \(\mathrm{Im}(P)\), i.e.,
\begin{equation}
X = XP + X(I-P),
\end{equation}
and
\begin{equation}
XP\Delta W_1 = X\Delta W_1.
\end{equation}

Equivalently,
\begin{equation}
X(I-P)\Delta W_1 = 0.
\end{equation}

\paragraph{Interference Definition.}
After jointly applying the two updates,
\begin{equation}
W'
=
W + \Delta W_1 + \widetilde{\Delta W}_2,
\end{equation}
we define the induced interference on functionality~1 as
\begin{equation}
E_1
=
XP\widetilde{\Delta W}_2,
\end{equation}
and the removed component of functionality~2 caused by the orthogonality constraint as
\begin{equation}
E_2
=
X\Delta W_2P_k.
\end{equation}

\begin{theorem}[Orthogonal Subspace Decoupling]
Under the above setup, the following statements hold.

\textbf{(1) Functional Interference Bound.}
The interference induced on functionality~1 satisfies
\begin{equation}
\|E_1\|_2
\le
\sigma_{k+1}(XP)\,
\|\Delta W_2\|_2,
\end{equation}
where \(\sigma_{k+1}(XP)\) denotes the \((k+1)\)-th singular value of \(XP\).

\vspace{0.5em}

\textbf{(2) Capacity Loss Bound.}
The removed component of functionality~2 satisfies
\begin{equation}
\|E_2\|_F^2
\le
\|X\|_2^2
\|\Delta W_2P_k\|_F^2.
\end{equation}

Moreover,
\begin{equation}
\|\widetilde{\Delta W}_2\|_F^2
=
\|\Delta W_2\|_F^2
-
\|\Delta W_2P_k\|_F^2.
\end{equation}

Therefore, the removed energy is exactly characterized by the overlap between \(\Delta W_2\) and the dominant singular subspace of \(\Delta W_1\).
\end{theorem}\label{thm:1}

\begin{proof}
By definition,
\begin{equation}
E_1
=
XP\Delta W_2(I-P_k).
\end{equation}

Since \(P_k\) projects onto the top-\(k\) right singular subspace, the Eckart--Young theorem gives
\begin{equation}
\|XP(I-P_k)\|_2
=
\sigma_{k+1}(XP).
\end{equation}

Thus,
\begin{align}
\|E_1\|_2
&=
\|XP(I-P_k)\Delta W_2(I-P_k)\|_2
\\
&\le
\|XP(I-P_k)\|_2
\|\Delta W_2\|_2
\\
&=
\sigma_{k+1}(XP)\|\Delta W_2\|_2.
\end{align}

For the second part,
\begin{equation}
E_2
=
X\Delta W_2P_k.
\end{equation}

Applying the standard inequality
\begin{equation}
\|AB\|_F
\le
\|A\|_2\|B\|_F,
\end{equation}
we obtain
\begin{equation}
\|E_2\|_F^2
\le
\|X\|_2^2
\|\Delta W_2P_k\|_F^2.
\end{equation}

Finally, since
\begin{equation}
\widetilde{\Delta W}_2
=
\Delta W_2(I-P_k),
\end{equation}
and
\begin{equation}
P_k(I-P_k)=0,
\end{equation}
the Pythagorean decomposition gives
\begin{equation}
\|\Delta W_2\|_F^2
=
\|\Delta W_2P_k\|_F^2
+
\|\widetilde{\Delta W}_2\|_F^2.
\end{equation}

This completes the proof.
\end{proof}

\paragraph{Discussion.}
The theorem reveals a clear trade-off between compositionality and functional capacity. If the first functionality is approximately low-rank, then
\begin{equation}
\sigma_{k+1}(XP)\approx 0,
\end{equation}
implying that the interference induced by the second functionality is tightly bounded. Meanwhile, the capacity reduction of the second functionality depends only on its overlap with the dominant singular subspace of the first functionality. 

Intuitively, Theorem~\ref{thm:1} indicates that the compositional error in functionality 1 is upper-bounded by the singular values of the minor components, which are typically small in practice, while the capacity loss for functionality 2 is upper-bounded by those of the major components. 
Fortunately, the update components for functionality 2 are trainable in our setup. We can explicitly enforce the model to adapt to functionality 2 without affecting the major components of functionality 1, thereby ensuring minimal interference between the two. 

\paragraph{Output Compositionality Analysis.}

We further analyze how orthogonal subspace projection improves the additivity of model outputs under nonlinear transformations.

Let
\begin{equation}
f:\mathbb{R}^{d\times m}\rightarrow\mathbb{R}^p
\end{equation}
denote the nonlinear model mapping from the parameter space to the output space.

Consider two parameter updates
\begin{equation}
\Delta W_1,
\quad
\widetilde{\Delta W}_2
=
\Delta W_2(I-P_k),
\end{equation}
where \(P_k=V_kV_k^\top\) is the projector onto the dominant right-singular subspace of \(\Delta W_1\).

We define the compositionality error as
\begin{equation}
E_{\mathrm{comp}}
=
f(W+\Delta W_1+\widetilde{\Delta W}_2)
-
f(W+\Delta W_1)
-
f(W+\widetilde{\Delta W}_2)
+
f(W).
\end{equation}

This quantity measures the deviation from ideal additive composition.

\begin{theorem}[Orthogonality-Induced Compositionality]
Assume \(f\) is twice continuously differentiable in a neighborhood of \(W\), and let
\begin{equation}
H_W
=
\nabla_W^2 f(W)
\end{equation}
denote the Hessian operator at \(W\).

Then the compositionality error satisfies
\begin{equation}
E_{\mathrm{comp}}
=
\langle
\Delta W_1,
H_W\widetilde{\Delta W}_2
\rangle
+
o\!\left(
\|\Delta W_1\|_F
\|\widetilde{\Delta W}_2\|_F
\right).
\end{equation}

Furthermore, if the Hessian matrix $H_W$ is diagonal and 
\begin{equation}
\widetilde{\Delta W}_2
=
\Delta W_2(I-P_k),
\end{equation}
then
\begin{equation}
\|\Delta W_1^\top \widetilde{\Delta W}_2\|_F
\le
\sigma_{k+1}(\Delta W_1)
\,
\|\Delta W_2\|_F,
\end{equation}
which yields
\begin{equation}
\|E_{\mathrm{comp}}\|
\le
\|H_W\|_2
\,
\sigma_{k+1}(\Delta W_1)
\,
\|\Delta W_2\|_F
+
o(\cdot).
\end{equation}

Therefore, the compositionality error is upper bounded by the residual tail singular value outside the dominant singular subspace of the first functionality.
\end{theorem}

\begin{proof}
Using the second-order Taylor expansion around \(W\),
\begin{align}
f(W+\Delta W)
&=
f(W)
+
\langle J_W,\Delta W\rangle
+
\frac12
\langle \Delta W,
H_W \Delta W
\rangle
+
o(\|\Delta W\|^2),
\end{align}
where \(J_W=\nabla_W f(W)\).

Substituting
\begin{equation}
\Delta W
=
\Delta W_1+\widetilde{\Delta W}_2,
\end{equation}
we obtain
\begin{align}
f(W+\Delta W_1+\widetilde{\Delta W}_2)
=
&\,
f(W)
+
\langle J_W,\Delta W_1\rangle
+
\langle J_W,\widetilde{\Delta W}_2\rangle
\\
&+
\frac12
\langle \Delta W_1,
H_W\Delta W_1
\rangle
\\
&+
\frac12
\langle \widetilde{\Delta W}_2,
H_W\widetilde{\Delta W}_2
\rangle
\\
&+
\langle \Delta W_1,
H_W\widetilde{\Delta W}_2
\rangle
+
o(\cdot).
\end{align}

Subtracting
\begin{equation}
f(W+\Delta W_1)
+
f(W+\widetilde{\Delta W}_2)
-
f(W),
\end{equation}
all first-order and self-quadratic terms cancel, yielding
\begin{equation}
E_{\mathrm{comp}}
=
\langle
\Delta W_1,
H_W\widetilde{\Delta W}_2
\rangle
+
o(\cdot).
\end{equation}

Finally, if $H_W$ is diagonal, since
\begin{equation}
\widetilde{\Delta W}_2
=
\Delta W_2(I-P_k),
\end{equation}
and \(P_k\) projects onto the dominant right-singular subspace of \(\Delta W_1\), the residual overlap satisfies
\begin{equation}
\|\Delta W_1^\top \widetilde{\Delta W}_2\|_F
\le
\sigma_{k+1}(\Delta W_1)
\,
\|\Delta W_2\|_F.
\end{equation}

Substituting this bound completes the proof.

\end{proof}

Intuitively, the Hessian matrix $H_W$ characterizes output interference between the two functionalities. 
According to recent studies~\cite{liu2022compositional}, large-scale diffusion models inherently support functional composition, indicating the potential diagonality of $H_W$. 
Moreover, in our practical setup, functionality 1 is responsible for diffusion timestep modulation, while functionality 2 handles spatial control, further suggesting their conceptual independence. 

\section{Additional implementation details}\label{appendix:implementation}
\subsection{Training Details}
\label{appendix:training}

\noindent \textbf{Model Configuration.}
We build SVEET upon Wan2.1-1.3B-VACE~\cite{jiang2025vace} as the
bidirectional editing backbone.
Unless otherwise specified, chunk-wise Causal Forcing~\cite{zhu2026causal}
is used as the default streaming backbone.
The pretrained parameters of both backbones remain frozen throughout training,
and only the VACE control branch is optimized.
Specifically, we insert rank-128 LoRA adapters
~\cite{hu2021loralowrankadaptationlarge} into the $q$, $k$, $v$, $o$,
$\mathrm{ffn}.0$, and $\mathrm{ffn}.2$ linear layers of the control branch.
The original spatiotemporal self-attention in the VACE branch is replaced with
frame-wise 2D spatial attention, such that condition features are extracted
independently for each source frame.

\noindent \textbf{Orthogonal Projector Construction.}
For Orthogonal Decoupled Training (ODT), we first construct a small calibration
set including over 500 samples and feed identical video--prompt pairs to the bidirectional and streaming backbones.
At each DiT block $\ell$, we collect the corresponding hidden features
$\mathbf{X}_b^{(\ell)}$ and $\mathbf{X}_c^{(\ell)}$ and estimate the
bidirectional-to-causal linear transformation using the ridge-regression
formulation described in Sec.~\ref{sec:theory}.
We then perform SVD on the residual transformation
$\mathbf{A}_1^{(\ell)}$.
The rank $k$ is selected independently for each layer such that the top-$k$
singular directions capture $80\%$ of the cumulative discrepancy energy.
The resulting orthogonal projector
$\boldsymbol{\Pi}_{\perp}^{(\ell)}$ is computed once before training and
remains fixed thereafter.
During optimization, the learnable LoRA update is projected onto this
orthogonal complement as described in Eq.~\ref{eq:reparam}.

\noindent \textbf{Training Protocol.}
All models are optimized using AdamW~\cite{loshchilov2017decoupled} with a
learning rate of $5\times10^{-5}$ and weight decay of $10^{-2}$.
We use a batch size of 1 and train each editing task for 10 epochs on a single
NVIDIA A100 GPU with 80GB memory.
Training videos contain 81 frames at a spatial resolution of
$480\times832$.
We train separate task-specific control branches for style transfer,
video inpainting, and depth-to-video generation, while keeping the underlying
bidirectional backbone unchanged.

\noindent \textbf{Streaming Transfer and Inference.}
After training, the learned control branch is directly attached to the frozen
streaming backbone without any additional fine-tuning or distillation.
No parameters of the streaming model are updated during this transfer.
Since the control branch uses temporally independent 2D attention, its
conditioning features can be computed frame by frame without introducing an
additional temporal KV cache.
For experiments on alternative streaming backbones, we keep the trained SVEET
control branch unchanged and replace only the streaming model used at
inference time; further results are reported in
Sec.~\ref{appendix:results1}.

\subsection{Dataset Details}
\label{appendix:dataset}

\begin{figure}[H]
    \centering
    \includegraphics[width=\linewidth]{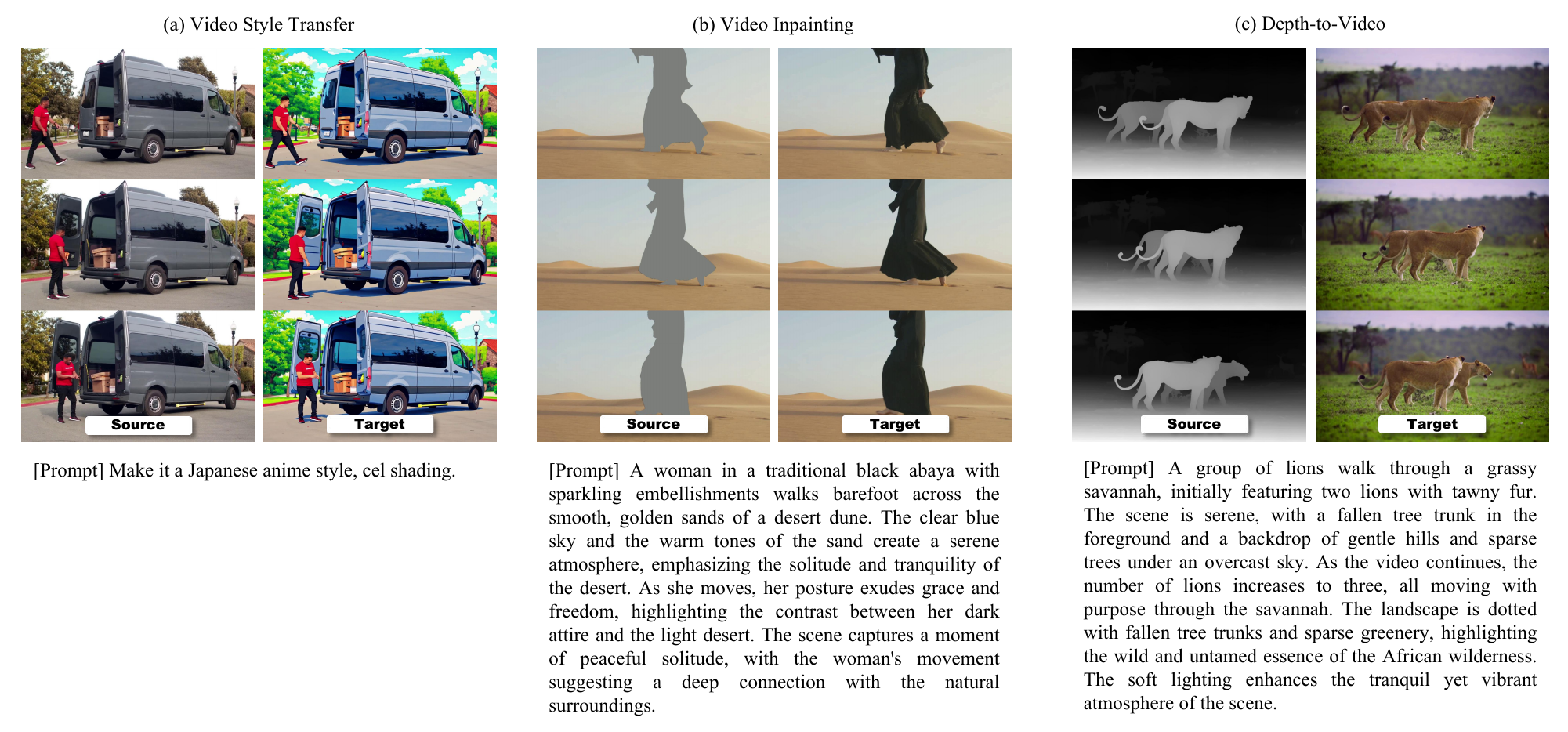}
    \caption{Representative training samples for three tasks.}
    \label{fig:dataset_examples}
    \vspace{-0.2cm}
\end{figure}

We construct separate datasets for the three editing tasks considered in our
experiments: style transfer, video inpainting, and depth-to-video generation.
The training and evaluation splits are disjoint for all tasks.

\noindent \textbf{Style Transfer.}
For video style transfer, we sample source videos from
Ditto~\cite{bai2025scaling}.
Each training sample consists of a source video paired with a textual instruction
specifying the target visual style.
We use 26 distinct styles with 500 video samples per style, resulting in
approximately 13K training pairs in total.
For evaluation, we randomly select 120 held-out videos from Ditto that do not
overlap with the training set.

\noindent \textbf{Video Inpainting.}
For video inpainting, we sample source videos from
VPData~\cite{bian2025videopainter}, resulting in more than 6K training samples.
Each sample contains the target video together with its masked conditioning
input, where the inpainting masks follow the original mask annotations provided
by VPData.
We use 80 held-out VPData samples for evaluation.

\noindent \textbf{Depth-to-Video Generation.}
For depth-to-video generation, we also sample videos from
VPData~\cite{bian2025videopainter}, obtaining approximately 7K training samples.
For each video, frame-wise depth maps are extracted using
Video-Depth-Anything~\cite{chen2025video} and used as the structural condition
for reconstructing the corresponding target video.
The same depth extraction pipeline is applied to both training and evaluation
data to ensure consistent conditioning.
We randomly select 80 held-out videos for evaluation.

\noindent \textbf{Preprocessing.}
All videos are converted to 81-frame clips at a spatial resolution of
$480\times832$.
Task-specific conditioning signals, including masked videos and depth maps, are
temporally aligned with their corresponding target frames.
The resulting training sets contain approximately 13K samples for style transfer,
over 6K samples for video inpainting, and 7K samples for depth-to-video generation.
The evaluation sets contain 120 samples for style transfer and 80 samples each
for video inpainting and depth-to-video generation.
Representative training samples for the three tasks are shown in
Fig.~\ref{fig:dataset_examples}.

\section{Additional Experimental Results}
\label{appendix:results}

\subsection{Results on Other Streaming Models}
\label{appendix:results1}

\begin{figure}[H]
    \centering
    \includegraphics[width=\linewidth]{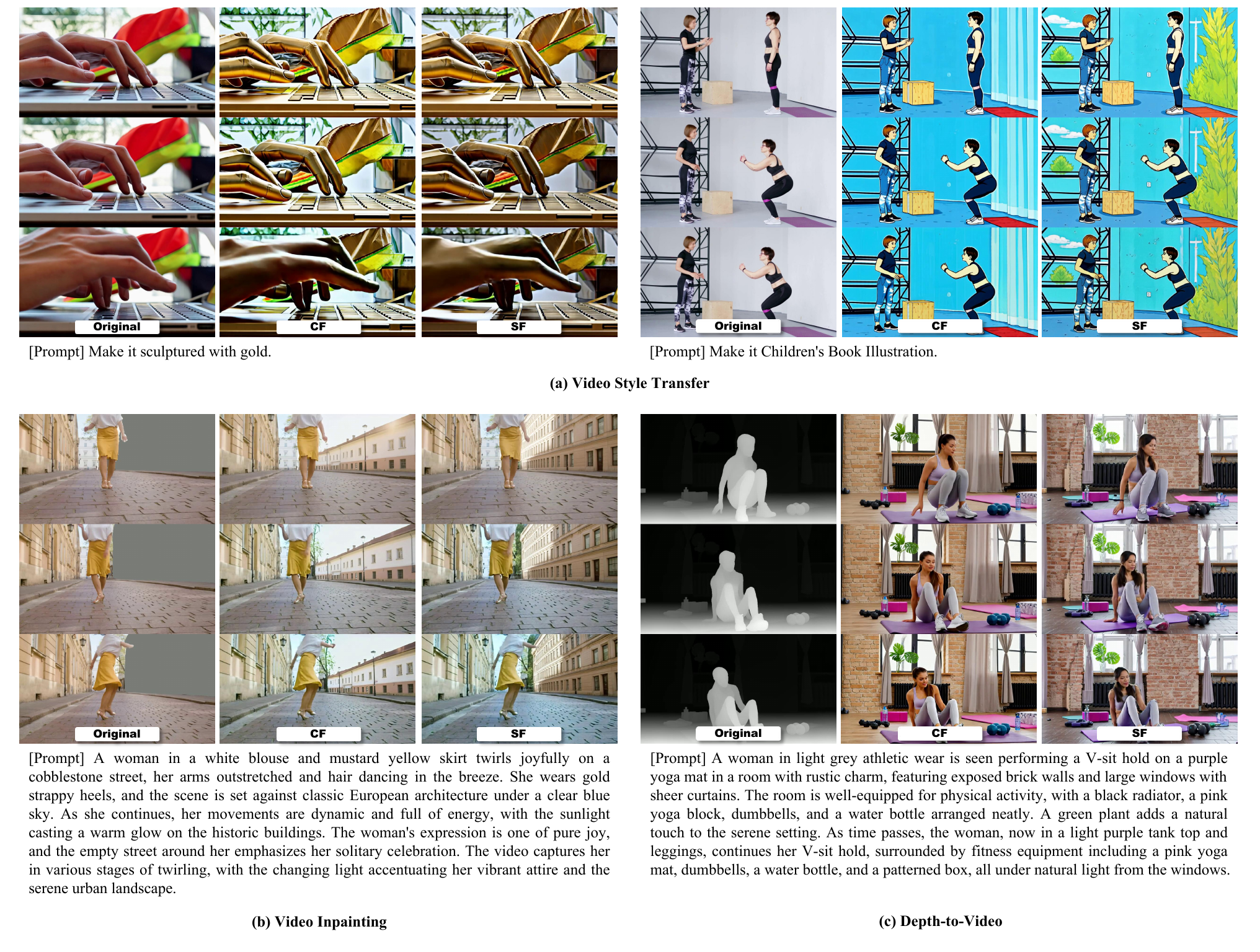}
    \caption{Transfer results on different streaming settings. Note that CF refers to Causal Forcing and SF refers to Self Forcing.}
    \label{fig:other_streaming}
    \vspace{-0.2cm}
\end{figure}

To further examine the transferability of our method across different streaming backbones, we additionally evaluate SVEET on Self Forcing~\cite{huang2025self}, in addition to the Causal Forcing backbone~\cite{zhu2026causal} used in the main experiments. Qualitative comparisons are shown in Fig.~\ref{fig:other_streaming}. Our method produces consistent editing effects on both Causal Forcing and Self Forcing, while preserving the main spatial structure and temporal dynamics of the source videos. These results indicate that the proposed bidirectional-to-streaming transfer strategy is not restricted to a single streaming backbone. This demonstrates the strong generalization capability and versatile applicability of our framework, revealing its great potential for universal video editing across diverse streaming-based generation architectures.

\subsection{More Visual Results}
\label{appendix:results2}


We further present additional results generated by our method in Fig.~\ref{fig:more_ours}, including diverse source videos, editing instructions, and control conditions. These examples demonstrate that SVEET can produce stable and visually consistent streaming editing results across a wide range of scenarios.

\begin{figure}[H]
    \centering
    \includegraphics[width=\linewidth]{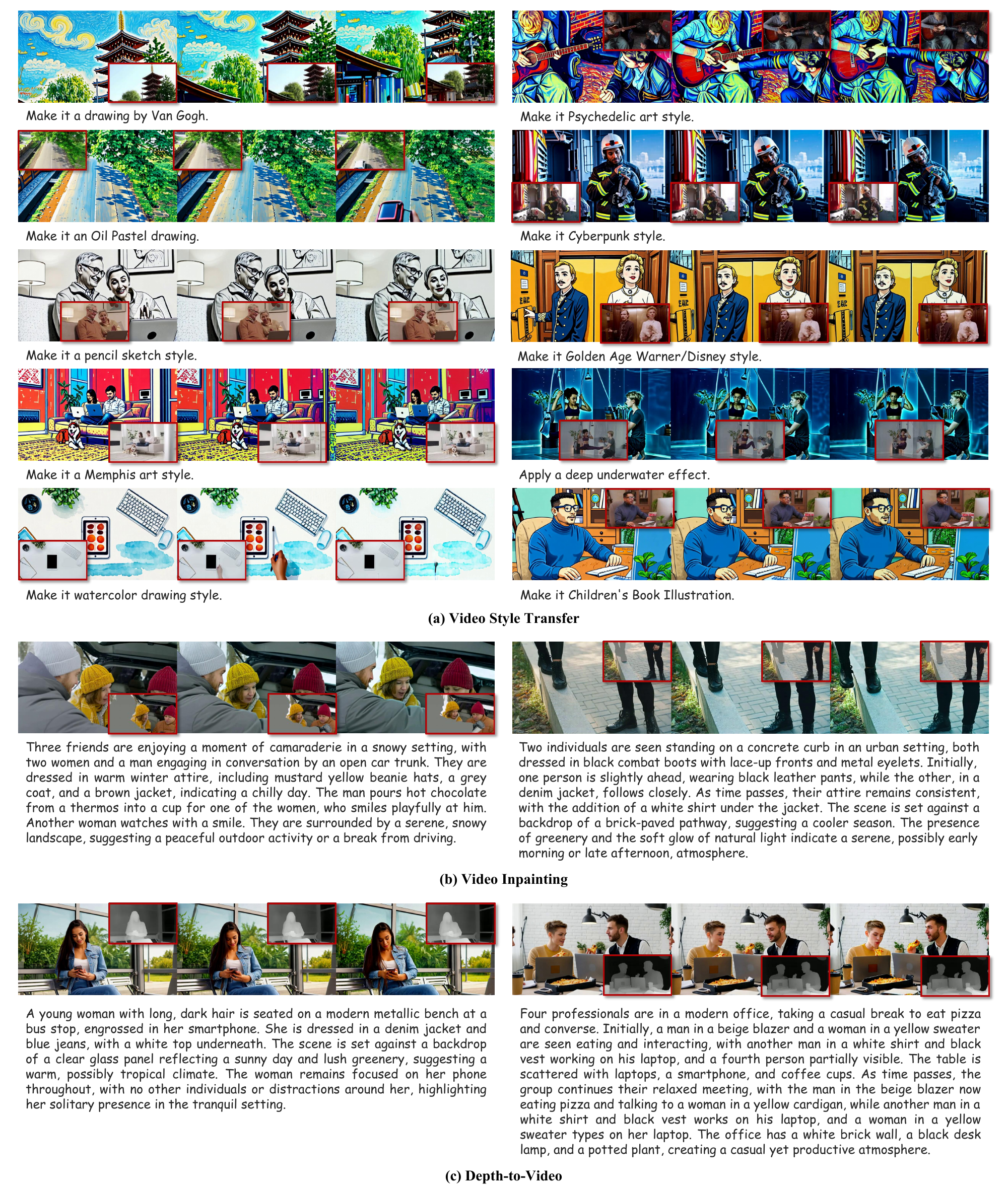}
    \caption{More visualization results.}
    \label{fig:more_ours}
    \vspace{-0.2cm}
\end{figure}

\end{document}